\documentclass[letterpaper]{article} 
\usepackage{aaai23}  
\usepackage{times}  
\usepackage{helvet}  
\usepackage{courier}  
\usepackage[hyphens]{url}  
\usepackage{graphicx} 
\usepackage{natbib}  
\usepackage{caption} 
\usepackage{algorithm}
\usepackage{algorithmic}

\usepackage{newfloat}
\usepackage{listings}
\DeclareCaptionStyle{ruled}{labelfont=normalfont,labelsep=colon,strut=off} 
\floatstyle{ruled}
\newfloat{listing}{tb}{lst}{}
\floatname{listing}{Listing}
\title{Disentangled Representation for Causal Mediation Analysis}
\author {
    Ziqi Xu \textsuperscript{\rm 1},
    Debo Cheng \textsuperscript{\rm 2,\rm 1,}\footnote{Corresponding author.},
    Jiuyong Li \textsuperscript{\rm 1,}\footnotemark[1],
    Jixue Liu \textsuperscript{\rm 1},
    Lin Liu \textsuperscript{\rm 1},
    Ke Wang \textsuperscript{\rm 3}
}
\affiliations {
    \textsuperscript{\rm 1} University of South Australia~
    \textsuperscript{\rm 2} Guangxi Normal University~
    \textsuperscript{\rm 3} Simon Fraser University\\
    \textsuperscript{\rm 1} ziqi.xu@mymail.unisa.edu.au, \{Jiuyong.li, Jixue.Liu, Lin.Liu\}@unisa.edu.au\\
    \textsuperscript{\rm 2} chengdb2016@gmail.com~
    \textsuperscript{\rm 3} wangk@cs.sfu.ca
}

\usepackage{array}
\usepackage{amsfonts}       
\usepackage{amsthm}
\usepackage{amsmath}
\usepackage{algorithm}
\usepackage{algorithmic}
\usepackage{theoremref}
\usepackage{xcolor}
\usepackage{titlesec}
\usepackage{multirow}
\usepackage{diagbox}
\usepackage{threeparttable}
\usepackage{verbatim}
\usepackage{booktabs}
\usepackage{subcaption}
\usepackage{setspace}
\usepackage{marvosym}

\newtheorem{definition}{Definition}

\newtheorem{assumption}{Assumption}

\newtheorem{theorem}{Theorem}

\makeatletter
\newcommand*{\indep}{%
	\mathbin{%
		\mathpalette{\@indep}{}%
	}%
}
\newcommand*{\nindep}{%
	\mathbin{
		\mathpalette{\@indep}{\not}
	}%
}
\newcommand*{\@indep}[2]{%
	\sbox0{$#1\perp\m@th$}
	\sbox2{$#1=$}
	\sbox4{$#1\vcenter{}$}
	\rlap{\copy0}
	\dimen@=\dimexpr\ht2-\ht4-.2pt\relax
	\kern\dimen@
	{#2}%
	\kern\dimen@
	\copy0 %
}

\begin{document}
\maketitle

\begin{abstract}
	Estimating direct and indirect causal effects from observational data is crucial to understanding the causal mechanisms and predicting the behaviour under different interventions. Causal mediation analysis is a method that is often used to reveal direct and indirect effects. Deep learning shows promise in mediation analysis, but the current methods only assume latent confounders that affect treatment, mediator and outcome simultaneously, and fail to identify different types of latent confounders (e.g., confounders that only affect the mediator or outcome). Furthermore, current methods are based on the sequential ignorability assumption, which is not feasible for dealing with multiple types of latent confounders. This work aims to circumvent the sequential ignorability assumption and applies the piecemeal deconfounding assumption as an alternative. We propose the Disentangled Mediation Analysis Variational AutoEncoder (DMAVAE), which disentangles the representations of latent confounders into three types to accurately estimate the natural direct effect, natural indirect effect and total effect. Experimental results show that the proposed method outperforms existing methods and has strong generalisation ability. We further apply the method to a real-world dataset to show its potential application.
\end{abstract}

\section{Introduction}
\label{sec:introduction}
Causal mediation analysis (CMA) aims to reveal how other attributes mediate the causal effect of a treatment on the outcome. It can help understand how the mechanism works and enable outcome prediction under various interventions in practice. In recent years, CMA has been widely used in operational decision~\cite{yin2019identification} and policy evaluation~\cite{ChengG022a}. For example, in the UC Berkeley admission process, it was found that female applicants were more likely to be rejected than male applicants~\cite{bickel1975sex}. However, investigators through CMA understood that the choice of disciplines, as a mediator was affected by gender and affected the acceptance rate. There was in fact no discrimination because female applicants preferred to apply for admission into disciplines with a higher level of competition and lower acceptance rates.  

The above example illustrates three essential ingredients of CMA, the natural direct effect (NDE), the natural indirect effect (NIE) and the total effect (TE), respectively. We use Figure~\ref{pic:intro1} to illustrate the above example in terms of the three types of effects. Here, TE measures the expected increase in the acceptance rate as gender changes from male to female; NDE measures the expected increase in acceptance rate as gender changes from male to female while setting the discipline choice to whatever value it would have attained prior to the change (i.e., under gender being male); NIE measures the expected increase in acceptance rate when gender is male, and discipline choice changes to whatever value it would have attained under gender being female~\cite{pearl2014interpretation}. To be fair for all applicants, the NDE should be zero since gender should not directly affect the acceptance rate, while the NIE is allowed to be non-zero if discipline choice is considered a non-sensitive attribute.
\begin{figure}[t]
	\centering
	\begin{subfigure}[b]{0.22\textwidth}
		\centering
		\includegraphics[scale=0.25]{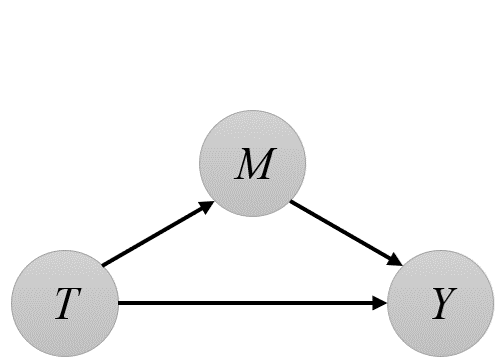}
		\caption{}
		\label{pic:intro1}
	\end{subfigure}
	\begin{subfigure}[b]{0.22\textwidth}
		\centering
		\includegraphics[scale=0.28]{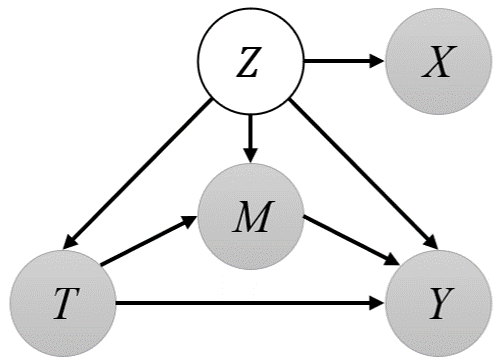}
		\caption{}
		\label{pic:intro3}
	\end{subfigure}
	\caption{(a) A causal graph for the UC Berkeley admission process example. ${T}$ denotes gender (treatment), ${M}$ denotes discipline choice (mediator), and ${Y}$ denotes acceptance rate (outcome). (b) Causal graph for CMAVAE~\cite{ChengG022a}. ${T}$ denotes treatment, ${M}$ denotes mediator, ${Y}$ denotes outcome, ${X}$ denotes proxy attributes and ${Z}$ denotes the representation of latent confounders which affect ${T}$, ${M}$ and ${Y}$ simultaneously.}
	\label{pic:intro}
\end{figure}

The major challenge in CMA is latent confounders. In the above example, an applicant's economic status affects his or her discipline choice and acceptance rate, and without measuring it, we cannot estimate NDE, NIE or TE. Fortunately, leveraging the strength of representation learning, researchers are able to learn the representations of latent confounders from other attributes. For example, the representation of the applicants' economic status can be learned from the collected information of their addresses and parental information. Deep learning has made significant progress in representation learning in the past decade. For example, \citet{KingmaW13} proposed Variational AutoEncoder (VAE) to learn representations with neural networks. \citet{LouizosSMSZW17} first combined causal inference and VAE and proposed CEVAE. CEVAE can learn the representation of latent confounders and inference individual treatment effect (ITE) and average treatment effect (ATE). \citet{ZhangLL21} proposed TEDVAE, which disentangles latent representations into three parts to achieve more accurate estimation of causal effects than CEVAE does. However, both CEVAE and TEDVAE do not consider mediator and cannot estimate NDE or NIE.

\citet{ChengG022a} proposed CMAVAE, which considers mediator and the estimation of NDE and NIE. It is the first work that combines deep learning and CMA, and CMAVAE assumes that the latent confounders affect treatment, mediator and outcome simultaneously as shown in Figure~\ref{pic:intro3}. Their work extended sequential ignorability~\cite{imai2010general} to representation learning and used proxy attributes ${X}$ to learn the representation of latent confounders. However, CMAVAE has limited generalisation ability since it does not consider different types of confounders, which, however, are often seen in practice. For example, the demand of the job market can be considered as a confounder that affects ${M}$ (discipline choice) and ${Y}$ (acceptance rate) in Figure~\ref{pic:intro1}, because the popularity of a job area affects students' choice of disciplines, the university may need to limit the acceptance rate due to the limitation of teaching resources. Furthermore, this confounder does not affect an applicant's gender. CMAVAE cannot handle this case since it assumes a confounder that affects treatment (gender in this example), mediator and outcome simultaneously.

\begin{figure}[t]
	\centering
	\includegraphics[scale=0.28]{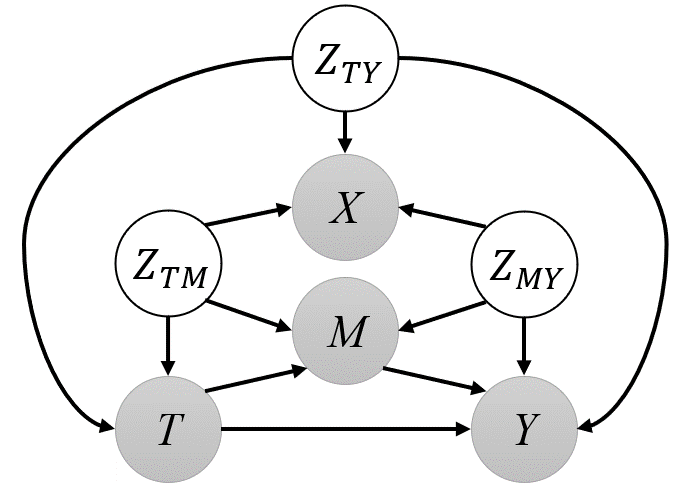}
	\caption{The causal graph for the proposed Disentangled Mediation Analysis Variational AutoEncoder (DMAVAE). ${T}$ is the treatment, ${M}$ is the mediator, ${Y}$ is the outcome and ${X}$ represents the set of the proxy attributes. ${Z_{TM}}$, ${Z_{TY}}$ and ${Z_{MY}}$ denote the disentangled representations of the three types of latent confounders.}
	\label{pic:intro2}
\end{figure}

In this paper, we consider a general case with three types of latent confounders as shown in Figure 2, where ${Z_{TM}}$ is the representation of the latent confounders that affect treatment and mediator; ${Z_{TY}}$ is the representation of the latent confounders that affect treatment and outcome; and ${Z_{MY}}$ is the representation of the latent confounders that affect mediator and outcome. We can access the observed proxy attributes ${X}$ and assume that the proxy attributes and latent confounders are inherently correlated~\cite{angrist2009mostly}. Our proposed disentanglement scheme in Figure~\ref{pic:intro2} follows the piecemeal deconfounding assumption~\cite{pearl2014interpretation}, which is a more relaxed assumption than sequential ignorability. We propose a novel method, namely Disentangled Mediation Analysis Variational AutoEncoder (DMAVAE) to learn the disentangled representations of the three types of latent confounders as illustrated in Figure~\ref{pic:intro2} by using a VAE-based approach, for estimating NDE, NIE and TE. We make the following contributions in this paper: 

\begin{itemize}
	\item We study a general case of CMA under the piecemeal deconfounding assumption, where there are three types of latent confounders.
	
	\item We propose a novel CMA method, DMAVAE, to learn the disentangled representations of the three types of latent confounders from proxy attributes to more accurately estimate NDE, NIE and TE.
	
	\item We evaluate the effectiveness of the DMAVAE method on synthetic datasets. Experiments show that DMAVAE outperforms existing CMA methods and has a strong generalisation ability. Furthermore, we conduct a case study to show the application scenarios of DMAVAE.
\end{itemize}

The rest of this paper is organised as follows. In Section 2, we discuss background knowledge on CMA. The details of DMAVAE are shown in Section 3. In Section 4, we discuss the experiment results. In Section 5, we discuss related works. Finally, we conclude the paper in Section 6.

\section{Background}
\label{sec:Background}
We use upper case letters to represent attributes and lower case letters to represent the values of the attributes. We follow the work in \cite{pearl2014interpretation} to introduce the notations, definitions and assumptions.

\subsection{Notations and Definitions}
Let ${T}_i$ represent the treatment status of the $i$-th individual, with ${T}_i = 1$ indicating the $i$-th individual receiving the treatment and ${T}_i = 0$ otherwise. We use ${M}_i$ to represent the mediator for the $i$-th individual, ${M}_i(t)$ for the value of the mediator when ${T}_i$ is fixed to $t$, where $t\in\{0,1\}$. Similarly, ${Y}_i$ is used to represent the outcome of the $i$-th individual and ${Y}_i(t,m)$ for the outcome when ${T}_i = t$ and ${M}_i = m$. Finally, we use ${X}_i$ to denote the set of proxy attributes. When the context is clear we omit the subscript $i$.
\begin{definition}[Total Effect~\cite{pearl2014interpretation}]
	\label{def:TE}
	\begin{equation}
		\begin{aligned}
			\mathrm{TE} &= \mathbb{E}[{Y}({1,M(1)}) - {Y}({0,M(0)})].
		\end{aligned}
	\end{equation}
\end{definition}

As stated in~\cite{pearl2014interpretation}, TE is the measure of ``the expected increase in the outcome ${Y}$ when the treatment changes from ${T} = 0$ to ${T} = 1$, while the mediator is allowed to track the change in ${T}$."
\begin{definition}[Natural Direct Effect~\cite{pearl2014interpretation}]
	\label{def:NDE}
	\begin{equation}
		\begin{aligned}
			\mathrm{NDE} = \mathbb{E}[{Y}({1,M(0)}) - {Y}({0,M(0)})].
		\end{aligned}
	\end{equation}
\end{definition}

Based on~\cite{pearl2014interpretation}, NDE is the measure of ``the expected increase in the outcome ${Y}$ that the treatment changes from ${T} = 0$ to ${T} = 1$, while setting the mediator to whatever value it would have attained (for each individual) prior to the change, i.e., under ${T} = 0$."
\begin{definition}[Natural Indirect Effect~\cite{pearl2014interpretation}]
	\label{def:NIE}
	\begin{equation}
		\begin{aligned}
			\mathrm{NIE} = \mathbb{E}[{Y}({0,M(1)}) - {Y}({0,M(0)})].
		\end{aligned}
	\end{equation}
\end{definition}

As explained in~\cite{pearl2014interpretation}, NIE is the measure of ``the expected increase in the outcome ${Y}$ when the treatment is held constant, at ${T} = 0$, and ${M}$ changes to whatever value it would have attained (for each individual) under ${T} = 1$."

Additionally, TE can be decomposed as follows:
\begin{equation}
	\label{equ:All}
	\begin{aligned}
		\mathrm{TE} = \mathrm{NDE} - \mathrm{NIE_r},
	\end{aligned}
\end{equation}where $\mathrm{NIE_r}$ denotes the reverse version of $\mathrm{NIE}$, as follows:
\begin{equation}
	\label{def:NIEr}
	\begin{aligned}
		\mathrm{NIE_r} = \mathbb{E}[{Y}({1,M(0)}) - {Y}({1,M(1)})].
	\end{aligned}
\end{equation}

This implies that $\mathrm{TE}$ is identifiable whenever $\mathrm{NDE}$ and $\mathrm{NIE_r}$ are identifiable~\cite{pearl2014interpretation}. Note that NDE and NIE are defined in counterfactual expressions. To allow these counterfactual expressions to be estimated, \citet{pearl2014interpretation} proposed three steps for the counterfactual derivation of NDE and NIE: First, we find an adjustment set under suitable assumptions; Second, we reduce all counterfactuals to do-expressions by using the adjustment set found; Finally, we identify the do-expressions from the data. 

\subsection{Assumptions}
\begin{assumption}[Sequential Ignorability~\cite{imai2010general}]
	\label{def:setA}
	There exists a set ${W}$ of measured attributes such that:
	\begin{enumerate}
		\item ${W}$ and ${T}$ deconfound the mediator-outcome relationship, keeping ${T}$ fixed: $[{Y}(t',m) \indep {M}(t) | {T}=t, {W}=w];$
		
		\item ${W}$ deconfounds the treatment-\{mediator, outcome\} relationship: $[{T} \indep ({Y}(t',m),{M}(t)) | {W}=w]$,
	\end{enumerate}
	where~{\small  $0<p({M}=m|{do}({T}=t,{W}=w))<1$, $0 < p({Y} =y|{do}({T}=t,{M}=m,{W}=w)) < 1$}, and $t,t' \in \{0,1\}$.
\end{assumption}

\citet{pearl2014interpretation} proposed the following assumption, which is more relaxed than Assumption~\ref{def:setA}. 

\begin{assumption}[Piecemeal Deconfounding~\cite{pearl2014interpretation}]
	\label{def:setC}
	There exists three sets of measured attributes ${W} = \{{W_1},{W_2},{W_3}\}$ such that:
	\begin{enumerate}
		\item No member of ${W_1}$ is affected by the treatment$;$
		
		\item ${W_1}$ deconfounds the ${M} \to {Y}$ relationship (holding ${T}$ constant)$;$
		
		\item $\{{W_2},{W_1}\}$ deconfounds the ${T} \to {M}$ relationship$;$
		
		\item $\{{W_3},{W_1}\}$ deconfounds the $\{{T},{M}\} \to {Y}$ relationship.
	\end{enumerate}
\end{assumption}

We note that Assumption~\ref{def:setC} follows the idea of divide-and-conquer, and allows the use of different sets of attributes to deconfound the mediator and outcome processes separately, rather than having to use the same set of deconfounding attributes as required by Assumption~\ref{def:setA}~\cite{pearl2014interpretation}. 

\section{Method}
\label{sec:Method}

\subsection{CMA via Representations}
As mentioned above, the piecemeal deconfounding assumption is more relaxed than the sequential ignorability assumption. Our goal is to extend the piecemeal deconfounding assumption to disentangled representation learning. For this, we propose the causal graph for the disentanglement as shown in Figure~\ref{pic:intro2}, and introduce the following assumption for the disentangled representations.
\begin{assumption}[Assumption for DMAVAE]
	\label{assumption:DMAVAE}
	~~~~~~
	\begin{enumerate}
		\item There exist three types of confounders, i.e., ${Z_{TM}}$, ${Z_{TY}}$ and ${Z_{MY}}$ as shown in Figure~\ref{pic:intro2}.
		
		\item There exist proxy attributes ${X}$ that approximate ${Z_{TM}}$, ${Z_{TY}}$ and ${Z_{MY}}$.
		
		\item $p({T},{M},{Y},{X},{Z_{TM}}, {Z_{TY}}, {Z_{MY}})$ can be recovered from all observed attributes $({T},{M},{Y},{X})$. 
	\end{enumerate}
\end{assumption}

For our proposed disentanglement, we have the following theorem.
\begin{theorem}
	\label{observation:DMAVAE}
	The disentangled representation ${Z_{MY}}$ following Figure~\ref{pic:intro2} satisfies Conditions 1 and 2 of Assumption~\ref{def:setC}.
\end{theorem}

\begin{proof}
	${Z_{MY}}$ is a non-descendant of ${T}$ as shown in Figure~\ref{pic:intro2}, hence Condition 1 of Assumption~\ref{def:setC} is satisfied (considering ${Z_{MY}}$ as ${W_1}$). All the back-door paths between ${M}$ and ${Y}$ when holding ${T}$ constant are blocked by ${Z_{MY}}$, i.e., ${M} \gets {Z_{MY}} \to {Y}$ is blocked by ${Z_{MY}}$, ${M} \gets {Z_{MY}} \to {X} \gets {Z_{TY}} \to {Y}$,  ${M} \gets {Z_{TM}} \to {X} \gets {Z_{TY}} \to {Y}$ and ${M} \gets {Z_{TM}} \to {X} \gets {Z_{MY}} \to {Y}$ are all blocked by $\emptyset$ since ${X}$ is a collider. Hence, Condition 2 of Assumption 2 is also satisfied when considering ${Z_{MY}}$ as ${W_1}$.
\end{proof}

Based on Theorem~\ref{observation:DMAVAE} and Equation 13 in~\cite{pearl2014interpretation}, the counterfactual expression of NDE (Equation~\ref{def:NDE}) and $\mathrm{NIE_r}$ (Equation~\ref{def:NIEr}) can be reduced to do-expressions as follows:
{\small \begin{equation}
		\label{NDE_do}
		\begin{aligned}
			\mathrm{NDE} =~ &[\mathbb{E}({Y}| {do}({T}=1,{M}=m), {Z_{MY}}={z_{MY}})\\ &-\mathbb{E}({Y}| {do}({T}=0,{M}=m), {Z_{MY}}={z_{MY}})] \\
			&\times p({M}=m|{do}({T}=0), {Z_{MY}}={z_{MY}})\\
			&\times p({Z_{MY}}={z_{MY}}),\\
		\end{aligned}
	\end{equation}
	\begin{equation}
		\label{NIE_do}
		\begin{aligned}
			\mathrm{NIE_r} =~ &\mathbb{E}({Y}| {do}({T}=1,{M}={m}), {Z_{MY}}={z_{MY}})\\
			&\times [p({M}={m}|{do}({T}=0), {Z_{MY}}={z_{MY}})\\ &- p({M}={m}|{do}({T}=1), {Z_{MY}}={z_{MY}})]\\
			&\times p({Z_{MY}}={z_{MY}}),
		\end{aligned}
\end{equation}}

Following Pearl's back-door adjustment formula~\cite{pearl2009causal, pearl2009causality}, we propose the following theorem which is used to reduce the do-expressions to probability expressions.
\begin{figure*}[t]
	\centering
	\includegraphics[scale=0.40]{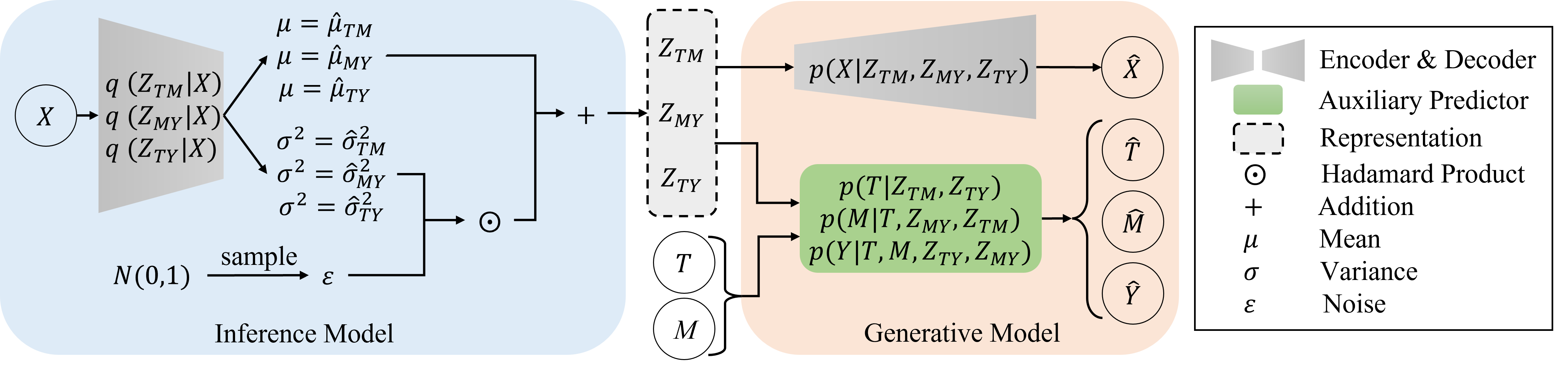}
	\caption{The architecture of Disentangled Mediation Analysis Variational AutoEncoder (DMAVAE).}
	\label{pic:DMAVAE}
\end{figure*}
\begin{theorem}
	\label{theory:1}If we can estimate $p({T},{M},{Y},{X},{Z_{TM}}, {Z_{TY}},\\{Z_{MY}})$ under the causal graph in Figure~\ref{pic:intro2}, $\mathrm{NDE}$ and $\mathrm{NIE_r}$ can be identified from data as follows:
	{\small \begin{equation}
			\label{equation:NDE}
			\begin{aligned}
				\mathrm{NDE} =~ &[\mathbb{E}({Y}| {T}=1,{M}={m}, {Z_{MY}}={z_{MY}}, {Z_{TY}}={z_{TY}})\\ &-\mathbb{E}({Y}| {T}=0,{M}={m}, {Z_{MY}}={z_{MY}}, {Z_{TY}}={z_{TY}})] \\
				&\times p({M}={m}|{T}=0, {Z_{MY}}={z_{MY}}, {Z_{TM}}={z_{TM}})\\
				&\times p({Z_{MY}}={z_{MY}}, {Z_{TM}}={z_{TM}}, {Z_{TY}}={z_{TY}}),\\
			\end{aligned}
		\end{equation}
		\begin{equation}
			\label{equation:NIE}
			\begin{aligned}
				\mathrm{NIE_r} =~ &\mathbb{E}({Y}| {T}=1,{M}={m}, {Z_{MY}}={z_{MY}}, {Z_{TY}}={z_{TY}})\\
				&\times [p({M}={m}|{T}=0, {Z_{MY}}={z_{MY}}, {Z_{TM}}={z_{TM}})\\ &- p({M}={m}|{T}=1, {Z_{MY}}={z_{MY}}, {Z_{TM}}={z_{TM}})]\\
				&\times p({Z_{MY}}={z_{MY}}, {Z_{TM}}={z_{TM}}, {Z_{TY}}={z_{TY}}).
			\end{aligned}
	\end{equation}}
\end{theorem}

\begin{proof}  
	We prove that Equations~\ref{NDE_do} and~\ref{NIE_do} can be identified, i.e., the do-expressions can be converted to the do-free expressions in Equations~\ref{equation:NDE} and~\ref{equation:NIE}, respectively. For this, we prove that $p({M}={m}|{do}({T}=0), {Z_{MY}}={z_{MY}})$ and $p({Y}=y| {do}({T}=0,{M}={m}), {Z_{MY}}={z_{MY}})$ are identifiable. The case for ${T}=1$ is identical and the expectation expressions with the do-operators are identifiable when the corresponding probability expressions are identifiable. Once $\mathrm{NDE}$ and $\mathrm{NIE_r}$ are identified, TE is calculated as shown in Equation~\ref{equ:All}.
	
	$p({M}={m}|{do}({T}=0), {Z_{MY}}={z_{MY}})$ is identifiable by adjusting ${Z_{TM}}$, since all the back-door paths between ${T}$ and ${M}$ when conditioning on ${Z_{MY}}$ are blocked by ${Z_{TM}}$, i.e., ${T} \gets {Z_{TM}} \to {M}$ is blocked by ${Z_{TM}}$, ${T} \gets {Z_{TM}} \to {X} \gets {Z_{TY}} \to {Y} \gets {M}$, ${T} \gets {Z_{TY}} \to {X} \gets {Z_{TM}} \to {M}$ and ${T} \gets {Z_{TY}} \to {Y} \gets {M}$ are blocked by $\emptyset$ since ${X}$ or ${Y}$ is a collider. Hence, $p({M}={m}|{do}({T}=0), {Z_{MY}}={z_{MY}}) = p({M}={m}|{T}=0, {Z_{MY}}={z_{MY}}, {Z_{TM}}={z_{TM}})$.
	
	$p({Y}=y| {do}({T}=0,{M}={m}), {Z_{MY}}={z_{MY}})$ is identifiable by adjusting ${Z_{TY}}$, since all the back-door paths between $\{{T},{M}\}$ and ${Y}$ when conditioning on ${Z_{MY}}$ are blocked by ${Z_{TY}}$, i.e., ${T} \gets {Z_{TY}} \to {Y}$ is blocked by ${Z_{TY}}$, ${T} \gets {Z_{TM}} \to {X} \gets {Z_{TY}} \to {Y}$ and ${M} \gets {Z_{TM}} \to {X} \gets {Z_{TY}} \to {Y}$ are blocked by $\emptyset$ since ${X}$ is a collider. Hence, $p({Y}=y| {do}({T}=0,{M}={m}), {Z_{MY}}={z_{MY}}) = p({Y}| {T}=0,{M}={m}, {Z_{MY}}={z_{MY}}, {Z_{TY}}={z_{TY}})$.
\end{proof}

\subsection{Learning Disentangled Representations}
In the previous discussion, we have assumed that ${X}$ contains (or captures) the information of the latent confounders, which can be disentangled into three types, i.e., ${Z_{TM}}$, ${Z_{TY}}$ and ${Z_{MY}}$. With these representations, we can estimate $\mathrm{NDE}$ and $\mathrm{NIE_r}$ by Equations~\ref{equation:NDE} and \ref{equation:NIE}, respectively. 

To learn these representations, we propose DMAVAE based on Variational Autoencoder (VAE)~\cite{KingmaW13}, which uses variational inference to learn representations. DMAVAE parameterises the causal graph in Figure~\ref{pic:intro2} as a model with neural network functions connecting the attributes of interest. The objective function of DMAVAE is the reconstruction error of the observed $({X},{T},{M},{Y}) $ and the inferred $(\hat{{X}},\hat{{T}},\hat{{M}},\hat{{Y}})$. The architecture of DMAVAE is shown in Figure~\ref{pic:DMAVAE}. We note that ${Y}$ is not shown in Figure~\ref{pic:DMAVAE} since ${Y}$ is not used as an input to the generative model.

In the inference model, we design three separate encoders $q({Z_{TM}} | {X})$, $q({Z_{MY}} | {X})$ and $q({Z_{TY}} | {X})$ that serve as variational posteriors over the three representations. The variational approximations of the posteriors are defined as:

~~~{\small $q({Z_{TM}}|{X}) = \prod_{j=1}^{D_{{Z_{TM}}}}\mathcal{N}(\mu = \hat{\mu}_{{TM}_j}, \sigma^2 = \hat{\sigma}^2_{{TM}_j})$};\\[0.02cm]

~~~{\small $q({Z_{MY}}|{X}) = \prod_{j=1}^{D_{{Z_{MY}}}}\mathcal{N}(\mu = \hat{\mu}_{{MY}_j}, \sigma^2 = \hat{\sigma}^2_{{MY}_j})$};\\[0.02cm]

~~~~{\small $q({Z_{TY}}|{X}) = \prod_{j=1}^{D_{{Z_{TY}}}}\mathcal{N}(\mu = \hat{\mu}_{{TY}_j}, \sigma^2 = \hat{\sigma}^2_{{TY}_j})$},\\[0.02cm]
where $\hat{\mu}_{{TM}},\hat{\mu}_{{MY}},\hat{\mu}_{{TY}}$ and $\hat{\sigma}^2_{{TM}},\hat{\sigma}^2_{{MY}},\hat{\sigma}^2_{{TY}}$ are the means and variances of the Gaussian distributions parameterised by neural networks. $D_{{Z_{TM}}}$,  $D_{{Z_{MY}}}$ and $D_{{Z_{TY}}}$ are the dimensions of ${Z_{TM}}$, ${Z_{TM}}$ and ${Z_{TY}}$, respectively.

The generative models for ${T}$ and ${X}$ are defined as:

~~~~~~~{\small $p({T}|{Z_{TM}}, {Z_{TY}}) = \mathrm{Bern}(\sigma(g({Z_{TM}}, {Z_{TY}})))$};\\[0.1cm]
{\small $p({X}|{Z_{TM}}, {Z_{TY}}, {Z_{MY}}) = \prod_{j=1}^{D_{{X}}} p({X}_j|{Z_{TM}}, {Z_{TY}}, {Z_{MY}})$},\\[0.1cm]
where $g(\cdot)$ is a neural network parameterised by its own parameters, $\sigma(\cdot)$ is the logistic function and $D_{{X}}$ is the dimension of ${X}$.

Specifically, the generative models for ${M}$ and ${Y}$ vary depending on the types of the attribute values. For continuous ${M}$ and ${Y}$, the models are defined as:
\begin{equation*}
	{\small \begin{aligned}
		p({M}|{T}, {Z_{TM}}, {Z_{MY}}) &= \mathcal{N}(\mu = \hat{\mu}_{{M}}, \sigma^2 = \hat{\sigma}^2_{{M}});\\
		\hat{\mu}_{{M}} = ({T} g({Z_{TM}}, {Z_{MY}}) &+ (1-{T})g({Z_{TM}}, {Z_{MY}}));\\
		\hat{\sigma}^2_{{M}} = ({T}g({Z_{TM}}, {Z_{MY}}) &+ (1-{T})g({Z_{TM}}, {Z_{MY}})),
	\end{aligned}	}
\end{equation*}
\begin{equation*}
	{\small \begin{aligned}
		p({Y}|{T}, {M}, {Z_{TY}}, {Z_{MY}}) &= \mathcal{N}(\mu = \hat{\mu}_{{Y}}, \sigma^2 = \hat{\sigma}^2_{{Y}});\\
		\hat{\mu}_{{Y}} = ({T}g({M}, {Z_{TY}}, {Z_{MY}}) &+ (1-{T})g({M}, {Z_{TY}}, {Z_{MY}}));\\
		\hat{\sigma}^2_{{Y}} = ({T}g({M}, {Z_{TM}}, {Z_{MY}}) &+ (1-{T})g({M}, {Z_{TM}}, {Z_{MY}})),
	\end{aligned}	}
\end{equation*}

For binary ${M}$ and ${Y}$ the models are defined as:
\begin{equation*}
	{\small \begin{aligned}
		p({M}|{T}, {Z_{TM}}, {Z_{MY}}) &= \mathrm{Bern}(\sigma(g({T}, {Z_{TM}}, {Z_{MY}})));\\
		p({Y}|{T}, {M}, {Z_{TY}}, {Z_{MY}}) &= \mathrm{Bern}(\sigma(g({T}, {M}, {Z_{TY}}, {Z_{MY}}))).
	\end{aligned}	}
\end{equation*}

Following the same setting in~\cite{KingmaW13}, we choose Gaussian distribution as the prior distributions of the representations, which are defined as: 

~~~~~~~~~~~~~~~~{\small $p({Z_{TM}}) = \prod_{j=1}^{D_{{Z_{TM}}}}\mathcal{N}(z_{{{TM}}_j} | 0,1)$};\\[0.05cm]

~~~~~~~~~~~~~~~~{\small $p({Z_{MY}}) = \prod_{j=1}^{D_{{Z_{MY}}}}\mathcal{N}(z_{{{MY}}_j} | 0,1)$};\\[0.05cm]

~~~~~~~~~~~~~~~~~{\small $p({Z_{TY}}) = \prod_{j=1}^{D_{{Z_{TY}}}}\mathcal{N}(z_{{{TY}}_j} | 0,1)$}.
\\

We can now form the evidence lower bound (ELBO) for the above inference and generative networks:
\begin{equation}
	{\small \begin{aligned}
		\mathcal{M} =~ &\mathbb{E}_{q}[\log p({X}|{Z_{TM}}, {Z_{TY}}, {Z_{MY}})] \\&- D_{\mathrm{KL}}[q({Z_{TM}}|{X})||p({Z_{TM}})] \\&- D_{\mathrm{KL}}[q({Z_{MY}}|{X})||p({Z_{MY}})] \\&- D_{\mathrm{KL}}[q({Z_{TY}}|{X})||p({Z_{TY}})],
		\label{eqa:ELBO1}
	\end{aligned}	}
\end{equation}where $\mathbb{E}_{q}$ stands for $\mathbb{E}_{q({Z_{TM}}|{X})q({Z_{MY}}|{X})q({Z_{TY}}|{X})}$. The first term denotes the reconstruction loss between ${X}$ and $\hat{{X}}$ ; each of the other terms is used to calculate the $\mathrm{KL}$ divergence between the prior knowledge and the learned representations.

Following the works in~\cite{LouizosSMSZW17, ZhangLL21}, we design three auxiliary predictors to encourage the disentanglement of the representations. To optimise the parameters in the auxiliary predictors, we add them to the loss function. Since maximising $\mathcal{M}$ is equal to minimising $-\mathcal{M}$, the final loss function of DMAVAE is defined as:
\begin{equation}
{\small 	\begin{aligned}
		\mathcal{L}_{\text{DMAVAE}} =~& -\mathcal{M} + \mathrm{log} q({T}|{Z_{TM},{Z_{TY}}}) \\&+ \mathrm{log} q({M}|{T}, \mathrm{Z_{TM},{Z_{MY}}}) \\&+ \mathrm{log} q({Y}|{T}, {M}, {Z_{TY},{Z_{MY}}}).
	\end{aligned}	}
\end{equation}

\section{Experiments}
\label{sec:Experiments}
Evaluating estimated causal effects has always been a huge challenge because we have no access to ground truth in practice~\cite{guo2020survey}. Evaluating the results of CMA is considered a more difficult task, as it requires accurate estimation of causal effects with respect to different paths. Therefore the evaluation of causal effect estimation and CMA mainly relies on simulated datasets~\cite{huber2016finite}. Following the works in~\cite{LouizosSMSZW17, ChengG022a}, we construct simulated datasets to evaluate the performance of DMAVAE and compare it with the baseline methods. We also evaluate the generalisation ability of DMAVAE. Finally, we apply DMAVAE to a real-world dataset and verify its usability. The code is available at \url{https://github.com/IRON13/DMAVAE}.

\subsection{Experiment Setup}
\subsubsection{Methods for Comparison}
\label{sec:Framework}
We compare DMAVAE with traditional and deep learning-based CMA methods. For traditional CMA methods, we select the commonly used parametric methods LSEM~\cite{baron1986moderator}, LSEM-I~\cite{imai2010general}, NEW-I, NEW-W~\cite{lange2012simple}, and semi-parametric method IPW~\cite{robins1994estimation,huber2014identifying}. For deep learning-based CMA methods, we compare our method with CMAVAE~\cite{ChengG022a}, which is the only VAE-based CMA method. Since there are no other deep learning-based CMA methods, we select two well-known causal effect estimators (CEVAE~\cite{LouizosSMSZW17} and TEDVAE~\cite{ZhangLL21}) as baselines for TE estimation only.

\subsubsection{Evaluation Criterion}
For evaluating the performance of DMAVAE and the baselines, we use the estimation bias $|(\hat{\beta} - \beta)/\beta|\times100\%$ as the metric, where $\hat{\beta}$ is the estimated results and $\beta$ is the ground truth. 

\subsubsection{Implementation Details}
We use PyTorch~\cite{NEURIPS2019_9015} and Pyro~\cite{bingham2018pyro} to implement DMAVAE. For LSME and LSME-I, we use their implementations in R package ``mediation"~\cite{tingley2014mediation}. For NEW-I and NEW-W, we use their implementations in the R package ``Medflex"~\cite{steen2017medflex}. For IPW, we use the implementation in the R package ``causal-weight"~\cite{bodory2018causalweight}. CEVAE\footnote{https://github.com/AMLab-Amsterdam/CEVAE} and  TEDVAE\footnote{https://github.com/WeijiaZhang24/TEDVAE} are open-source on GitHub, and CMAVAE is implemented in Pyro and PyTorch by us.

\begin{table*}[t]
	\centering
	\begin{subtable}[h]{0.95\textwidth}
		\centering
		{\small \begin{tabular}{cccccccccc}
				\specialrule{0.5pt}{0pt}{1pt}
				N      & 2k      & 3k      & 4k      & 5k      & 6k      & 7k      & 8k      & 9k      & 10k     \\ \specialrule{0.25pt}{-1.5pt}{1pt}
				LSEM   & \underline{7.1 ± 1.6} & 7.1 ± 1.3 & 7.3 ± 1.2 & 6.9 ± 1.3 & 7.4 ± 0.9 & 7.5 ± 0.9 & 7.2 ± 1.0 & 7.2 ± 0.7 & 7.4 ± 0.9 \\
				LSEM-I & 8.3 ± 1.8 & 7.9 ± 1.3 & 8.1 ± 1.3 & 7.8 ± 1.1 & 8.2 ± 1.0 & 8.6 ± 0.9 & 8.0 ± 1.0 & 8.1 ± 0.8 & 8.3 ± 0.8 \\
				NEW-I  & 10.9 ± 8.4 & 7.8 ± 6.7 & 7.9 ± 6.3 & 7.5 ± 5.5 & 7.6 ± 4.5 & 8.8 ± 5.9 & 6.8 ± 4.5 & 6.4 ± 4.2 & 8.2 ± 5.2 \\
				NEW-W  & 32.1 ± 12.0 & 21.5 ± 8.9 & 22.2 ± 7.7 & 20.3 ± 8.4 & 20.7 ± 9.6 & 16.8 ± 7.0 & 12.9 ± 5.0 & 18.6 ± 6.5 & 15.6 ± 6.2 \\
				IPW    & 7.5 ± 2.5 & 7.0 ± 1.7 & 7.7 ± 1.4 & 7.3 ± 1.4 & 7.5 ± 1.4 & 7.9 ± 1.4 & 7.3 ± 1.3 & 7.6 ± 1.0 & 8.0 ± 0.9 \\
				CMAVAE & 8.9 ± 5.0 & \underline{5.5 ± 2.3} & \underline{4.4 ± 2.2} & \underline{3.8 ± 2.0} & \underline{3.9 ± 1.9} & \underline{4.6 ± 1.7} & \underline{4.5 ± 2.1} & \underline{4.7 ± 1.4} & \underline{5.2 ± 2.0} \\
				DMAVAE & \textbf{4.2 ± 2.7} & \textbf{1.7 ± 1.3} & \textbf{1.7 ± 1.3} & \textbf{1.8 ± 1.4} & \textbf{1.8 ± 1.1} & \textbf{1.4 ± 0.8} & \textbf{1.7 ± 1.3} & \textbf{1.3 ± 1.0} & \textbf{1.5 ± 1.1} \\ \specialrule{0.5pt}{0pt}{-3pt}
		\end{tabular}}
		\caption{Estimation bias ($\%$) for the estimated $\mathrm{NDE}$.}
		\hfill
		\label{tab:NDE}
	\end{subtable}
	\begin{subtable}[h]{0.95\textwidth}
		\centering
		{\small \begin{tabular}{cccccccccc}
				\specialrule{0.5pt}{-4pt}{1pt}
				N      & 2k      & 3k      & 4k      & 5k      & 6k      & 7k      & 8k      & 9k      & 10k     \\ \specialrule{0.25pt}{-1.5pt}{1pt}
				LSEM   & \underline{11.0 ± 3.0} & 10.7 ± 2.2 & 10.6 ± 2.5 & 10.4 ± 2.7 & 11.0 ± 1.6 & 10.4 ± 1.9 & 10.8 ± 1.7 & 10.5 ± 1.8 & 11.1 ± 1.5 \\
				LSEM-I & 13.2 ± 3.1 & 12.1 ± 2.3 & 12.5 ± 2.4 & 12.2 ± 2.5 & 12.7 ± 1.8 & 12.1 ± 1.9 & 12.5 ± 1.6 & 12.2 ± 1.6 & 13.0 ± 1.5 \\
				NEW-I  & 11.7 ± 9.1 & 9.5 ± 7.6 & 10.0 ± 7.0 & 9.0 ± 6.6 & 10.5 ± 4.7 & 12.1 ± 5.6 & 9.1 ± 5.4 & 9.6 ± 4.1 & 11.2 ± 5.1 \\
				NEW-W  & 24.6 ± 9.3 & 18.7 ± 5.8 & 18.3 ± 7.2 & 20.8 ± 7.0 & 19.2 ± 6.4 & 16.6 ± 5.5 & 16.5 ± 5.7 & 13.5 ± 5.5 & 18.7 ± 6.4 \\
				IPW    & 58.0 ± 3.7 & 57.1 ± 3.6 & 58.2 ± 2.5 & 57.7 ± 2.5 & 57.7 ± 2.2 & 57.7 ± 1.9 & 57.6 ± 1.7 & 58.0 ± 2.0 & 58.1 ± 1.2 \\
				CMAVAE & 15.1 ± 8.7 & \underline{8.2 ± 5.0} & \underline{6.1 ± 4.1} & \underline{5.7 ± 3.5} & \underline{6.0 ± 2.9} & \underline{7.7 ± 3.7} & \underline{9.6 ± 3.3} & \underline{9.4 ± 3.3} & \underline{9.2 ± 3.4} \\
				DMAVAE & \textbf{6.4 ± 4.9} & \textbf{3.9 ± 2.6} & \textbf{5.1 ± 3.5} & \textbf{4.8 ± 3.1} & \textbf{4.2 ± 2.4} & \textbf{4.5 ± 2.5} & \textbf{4.3 ± 2.1} & \textbf{4.5 ± 2.3} & \textbf{4.6 ± 2.3} \\ \specialrule{0.5pt}{0pt}{-3pt}
		\end{tabular}}
		\caption{Estimation bias ($\%$) for the estimated $\mathrm{NIE_r}$.}
		\hfill
		\label{tab:NIEr}
	\end{subtable}
	\begin{subtable}[h]{0.95\textwidth}
		\centering
		{\small \begin{tabular}{cccccccccc}
				\specialrule{0.5pt}{-4pt}{1pt}
				N      & 2k      & 3k      & 4k      & 5k      & 6k      & 7k      & 8k      & 9k      & 10k     \\ \specialrule{0.25pt}{-1.5pt}{1pt}
				LSEM   & \textbf{1.1 ± 0.9} & \underline{1.0 ± 0.8} & 1.3 ± 0.8 & 1.1 ± 0.7 & \textbf{1.1 ± 0.6} & 1.4 ± 0.6 & \textbf{1.1 ± 0.5} & \underline{1.2 ± 0.5} & \textbf{1.1 ± 0.5} \\
				LSEM-I & \underline{1.2 ± 0.9} & 1.1 ± 0.7 & \textbf{1.2 ± 0.7} & 1.1 ± 0.6 & \underline{1.2 ± 0.7} & 1.5 ± 0.7 & 1.3 ± 0.6 & \underline{1.2 ± 0.5} & \textbf{1.1 ± 0.5} \\
				NEW-I  & 10.9 ± 8.6 & 8.2 ± 7.0 & 8.6 ± 6.4 & 7.9 ± 5.9 & 8.5 ± 4.6 & 9.9 ± 5.8 & 7.5 ± 4.8 & 7.3 ± 4.3 & 9.2 ± 5.2 \\
				NEW-W  & 27.9 ± 11.2 & 18.8 ± 7.5 & 19.0 ± 7.3 & 18.5 ± 7.9 & 18.3 ± 8.3 & 15.4 ± 5.6 & 11.6 ± 5.0 & 14.6 ± 5.7 & 14.9 ± 5.9 \\
				IPW    & 14.8 ± 1.4 & 14.8 ± 1.0 & 14.8 ± 0.6 & 14.9 ± 0.7 & 14.7 ± 0.7 & 14.5 ± 0.7 & 14.8 ± 0.5 & 14.8 ± 0.5 & 14.5 ± 0.5 \\
				CEVAE    & 24.1 ± 5.6 & 27.0 ± 4.3 & 28.1 ± 1.1 & 28.3 ± 1.0 & 28.1 ± 0.8 & 27.8 ± 0.8 & 27.9 ± 1.1 & 27.5 ± 0.8 & 27.7 ± 1.1 \\
				TEDVAE    & 28.3 ± 1.3 & 28.3 ± 1.3 & 27.9 ± 1.3 & 28.1 ± 0.9 & 27.9 ± 1.0 & 27.6 ± 0.8 & 28.3 ± 1.0 & 28.0 ± 0.9 & 27.9 ± 1.0 \\
				CMAVAE & \textbf{1.1 ± 0.9} & \textbf{1.0 ± 0.7} & \textbf{1.2 ± 0.7} & \underline{0.9 ± 0.7} & 1.2 ± 0.8 & \textbf{1.0 ± 0.6} & \underline{1.1 ± 0.8} & \textbf{0.9 ± 0.6} & \underline{1.1 ± 0.8} \\
				DMAVAE & \textbf{1.1 ± 0.9} & \textbf{1.0 ± 0.7} & \underline{1.3 ± 0.7} & \textbf{0.8 ± 0.7} & \underline{1.2 ± 0.7} & \underline{1.3 ± 0.9} & 1.3 ± 0.9 & 1.2 ± 0.8 & 1.2 ± 0.9 \\ \specialrule{0.5pt}{0pt}{-3pt}
		\end{tabular}}
		\caption{Estimation bias ($\%$) for the estimated TE.}
		\label{tab:TE}
	\end{subtable}
	\caption{Estimation bias for the estimated $\mathrm{NDE}$ ,$\mathrm{NIE_r}$ and $\mathrm{TE}$ on synthetic datasets. The best results are shown in \textbf{boldface}, and the runner-up results are \underline{underlined}.}
	\label{tab:Res}
\end{table*}

\subsection{Performance Evaluation} \label{sec:syn} We firstly generate synthetic datasets based on the causal graph in Figure~\ref{pic:intro2} (hence we know the ground truth) to evaluate the performance of our methods, in comparison with the baselines. We follow the same procedure in~\cite{LouizosSMSZW17, ChengG022a} to generate synthetic datasets with a variety range of sample sizes (denoted as $\mathrm{N}$): 2k, 3k, 4k, 5k, 6k, 7k, 8k, 9k and 10k.

To avoid the bias brought by data generation, we repeatedly generate 30 datasets for each sample size. The results for the mean and the standard deviation of estimation bias for $\mathrm{NDE}$, $\mathrm{NIE_r}$ and TE are shown in Table~\ref{tab:Res}. 

\subsubsection{Results} DMAVAE achieves the lowest estimation bias in both $\mathrm{NDE}$ and $\mathrm{NIE_r}$ compared with the baseline methods as shown in Tables~\ref{tab:NDE} and~\ref{tab:NIEr}. We note that DMAVAE performs clearly better than CMAVAE under the current data generation mechanism. A possible explanation is that since there are three types of latent confounders, but CMAVAE only learns one type of representation. As we mentioned above, different adjustment sets are needed for different estimation tasks, so only considering one type of representation can lead to a relatively large bias in estimating $\mathrm{NDE}$ and $\mathrm{NIE_r}$ in this case.

In Table~\ref{tab:TE}, LSEM, LSEM-I, CMAVAE and DMAVAE show better performance for estimating TE compared with the other baseline methods. With the increase in sample size, CMAVAE shows lower estimation bias than DMAVAE. The reason could be that estimating TE requires the representation of all confounders, as DMAVAE disentangles three types of representations, it may produce more cumulative errors compared to CMAVAE, which only assumes one type of latent confounders. 

\subsection{Evaluation on Generalisability}
We also conduct experiments on multiple cases to verify the effectiveness and the generalisation ability of DMAVAE to data generated from structures with different types of latent confounders. We note that among the baseline methods used previously, only CMAVAE can achieve a similar level of estimation bias as our methods for $\mathrm{NDE}$ and $\mathrm{NIE_r}$. Hence, we select CMAVAE as the compared model in this part of the evaluation.

We have already compared DMAVAE and CMAVAE under the latent confounder assumption shown in Figure~\ref{pic:intro2}. We have seen that our proposed DMAVAE method outperforms CMAVAE in estimating $\mathrm{NDE}$ and $\mathrm{NIE_r}$ and achieves similar performance in estimating $\mathrm{TE}$. In this section, we firstly compare DMAVAE and CMAVAE with the data generated based on the causal graph shown in Figure~\ref{pic:intro3}. The results are shown in Table~\ref{tab:compare}. We note that DMAVAE has a higher estimation bias compared to the results in Table~\ref{tab:Res}, while CMAVAE achieves a lower estimation bias than that in Table~\ref{tab:Res}. Such results are reasonable since the data generation mechanism fully complies with the assumption of CMAVAE (i.e., there only exists one type of confounder that affects ${T}$, ${M}$ and ${Y}$ simultaneously). However, we find that even under this assumption, DMAVAE achieves lower estimation bias in some datasets compared to CMAVAE, which implies that DMAVAE has good generalisation ability.

Furthermore, we consider multiple cases when comparing DMAVAE and CMAVAE in terms of their generalisability. Similar to the job market example mentioned in Section 1, the causal graph for real applications may contain one type or two types or all the three types of latent confounders. When considering the different combinations of the three types of latent confounders, we get six cases (indicated in Figure~\ref{pic:Ablation} as Case 1 to Case 6). The datasets for each case are generated separately following the respective causal graphs. For Case 1 to Case 3, we assume only one type of latent confounders exists. For Case 4 to Case 6, we assume there exist two types of latent confounders. The results for the six cases are shown in Figure~\ref{pic:Ablation}. From the figure, we see that DMAVAE outperforms CMAVAE in all 6 cases. This indicates that DMAVAE has better generalisation ability, and can deal with different cases in practice more effectively than CMAVAE.

\begin{table}[t]
	\centering
	{\small \begin{tabular}{ccccc}
			\specialrule{0.5pt}{-4pt}{1pt}
			Model & \multicolumn{2}{c}{DMAVAE} & \multicolumn{2}{c}{CMAVAE} \\ \specialrule{0.25pt}{-1.5pt}{1pt}
			N     & $\mathrm{NDE}$          & $\mathrm{NIE_r}$         & $\mathrm{NDE}$          & $\mathrm{NIE_r}$         \\ \specialrule{0.25pt}{-1.5pt}{1pt}
			2k    & 8.5 ± 4.8    & \textbf{7.4 ± 5.5}   & \textbf{5.2 ± 4.4}    & 19.9 ± 8.8   \\
			3k    & 4.6 ± 3.6    & \textbf{5.4 ± 3.5}   & \textbf{2.6 ± 2.3}    & 18.7 ± 9.3   \\
			4k    & 4.3 ± 2.4    & \textbf{4.5 ± 3.5}   & \textbf{2.9 ± 2.1}    & 9.9 ± 5.8   \\
			5k    & 4.3 ± 2.1    & \textbf{4.5 ± 2.6}   & \textbf{3.8 ± 2.2}    & 4.8 ± 1.0   \\
			6k    & \textbf{4.2 ± 2.0}    & \textbf{4.8 ± 3.5}   & 4.3 ± 2.5    & 4.8 ± 4.1   \\
			7k    & \textbf{3.6 ± 1.9}    & \textbf{4.9 ± 3.2}   & 3.8 ± 2.0    & 5.4 ± 3.8   \\
			8k    & 3.9 ± 1.7    & \textbf{4.6 ± 2.6}   & \textbf{3.5 ± 1.7}    & 4.7 ± 4.4   \\
			9k    & 4.1 ± 1.8    & 4.7 ± 2.8   & \textbf{3.9 ± 1.8}    & \textbf{3.9 ± 3.9}   \\
			10k   & \textbf{4.0 ± 1.9}    & 4.7 ± 3.0   & \textbf{4.0 ± 1.8}    & \textbf{3.8 ± 3.2}   \\ \specialrule{0.5pt}{0pt}{-3pt}
	\end{tabular}}
	\caption{Estimation bias ($\%$) for the estimated $\mathrm{NDE}$ and $\mathrm{NIE_r}$ with data generated based on the causal graph shown in Figure~\ref{pic:intro3}. The best results are shown in \textbf{boldface}.}
	\label{tab:compare}
\end{table}

\begin{figure}[t]
	\includegraphics[scale=0.19]{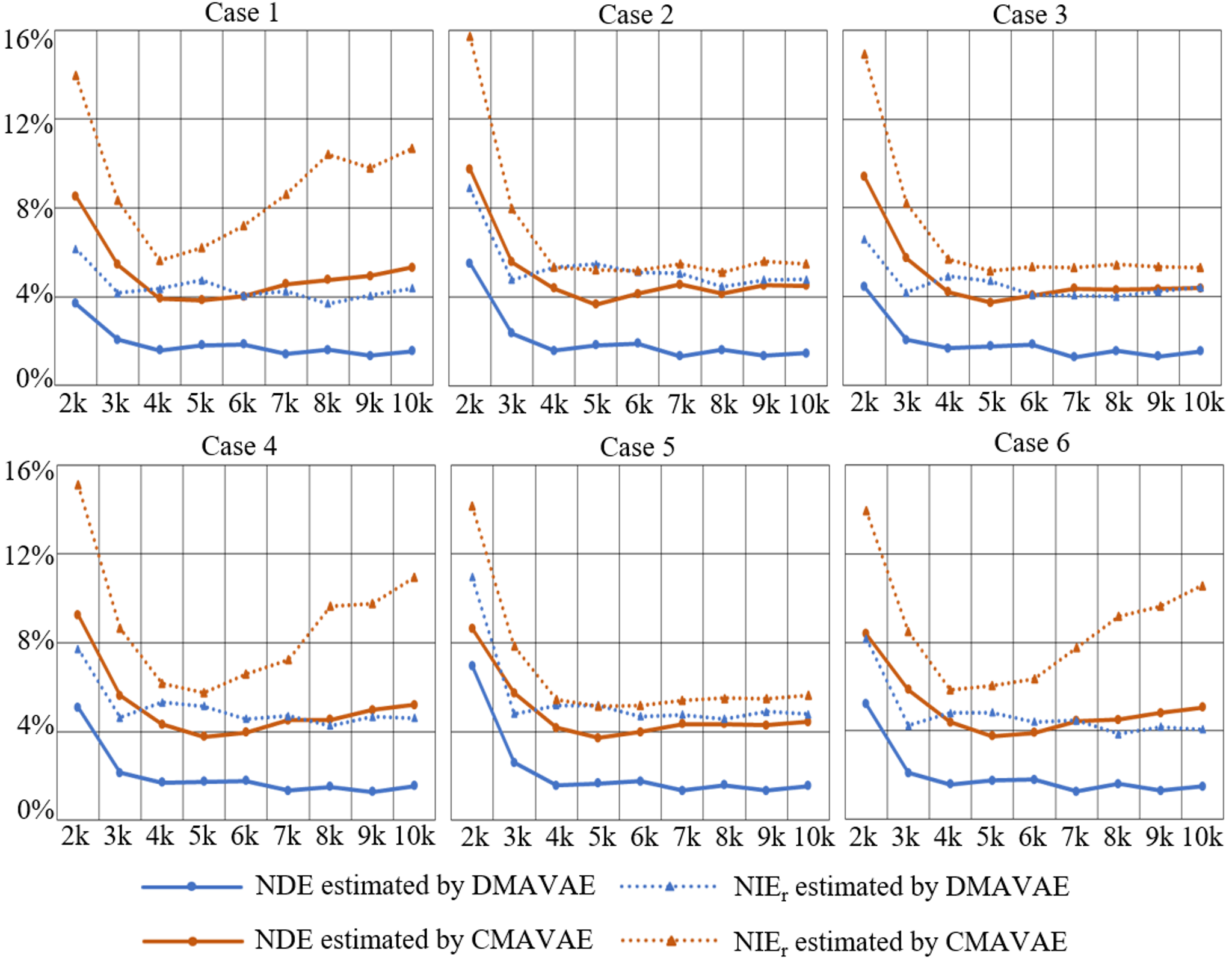}
	\caption{Results for the six cases. The horizontal axis indicates sample size ($\mathrm{N}$) and the vertical axis indicates estimation bias ($\%$).}
	\label{pic:Ablation}
\end{figure}

\subsection{Case Study}
\subsubsection{Adult Dataset} 
To illustrate the potential of DMAVAE for real-world applications, we show that DMAVAE can be used for detecting discrimination. We apply DMAVAE on the real-world dataset Adult. The dataset is retrieved from the UCI repository ~\cite{asuncion2007uci} and it contains 14 attributes including personal, educational and economic information for 48842 individuals. We use the sensitive attribute gender as ${T}$, occupation as ${M}$, income as ${Y}$ and all the other attributes as the proxy attributes ${X}$.

For causality-based discrimination detection, $\mathrm{NDE}$ indicates the causal effect that the sensitive attribute directly has on the outcome and $\mathrm{NIE}$ refers to the causal effect through the mediator on the outcome. Hence the approach detects discrimination in a dataset by estimating $\mathrm{NDE}$ and $\mathrm{NIE}$. There is direct discrimination when the NDE is greater than a given threshold $\tau$, and there exists indirect discrimination when the mediator is also considered as a sensitive attribute and the $\mathrm{NIE}$ is greater than $\tau$~\cite{zhang2018causal}. 

We obtain that $\mathrm{NDE} = 0.804$ and $\mathrm{NIE} = 0.002$ by applying DMAVAE on the Adult dataset ($\tau = 0.05$). The results indicate that there is significant direct discrimination against gender, while the discrimination against gender through occupation is ignorable.

\section{Related Work}
\label{sec:RelatedWork}
Mediation analysis has its origins in structural equation models (SEMs), and the specific methods of CMA can be traced back to path analysis~\cite{wright1934method}, which has been widely used in social sciences for decades~\cite{baron1986moderator,bollen1989structural}. Under the linear models, causal effects are expressed as the coefficients of structural equations. These coefficients are defined, identified, and estimated based on the assumption of sequential ignorability and commitment to a specific distribution. Specifically, LSEM is a well-established method based on linear SEMs for mediation analysis~\cite{baron1986moderator,james1983causal,judd1981process,mackinnon2012introduction,mackinnon1993estimating}. Although LSEM is extremely simple to use, the additional assumptions are often not satisfied in reality, leading to biased estimation~\cite{huber2016finite,rudolph2019causal}. 

The above situation has been improved in recent decades. Counterfactual thinking in statistics~\cite{rubin1974estimating,holland1988causal} and epidemiology~\cite{robins1992identifiability}, together with the formal definitions based on non-parametric structural equations provide theoretical support for mediation analysis to be expanded from linear to nonlinear models. With counterfactual thinking, defining direct and indirect effects requires no commitment to the distributional forms, and thus applies under functions with arbitrary nonlinear interactions, and both continuous and categorical variables~\cite{pearl2014interpretation, rubin2005causal}. For example, \citet{imai2010general} extended LSEM and proposed a general approach that brings the definition, identification, estimation, and sensitivity analysis of causal mediation effects closely together within a single method; \citet{huber2014identifying} demonstrated that the applicability of inverse probability weighting depends on whether some confounders are themselves influenced by the treatment or not. For more details on mediation analysis, please refer to the papers~\cite{mackinnon2007mediation, vanderweele2016mediation}.

In recent years, deep learning has demonstrated powerful ability in fitting non-linear models, and researchers are devoted to estimating causal effects through deep learning techniques. CEVAE~\cite{LouizosSMSZW17} is the first method that uses VAE to learn the representation of latent confounders and adjust for confounding bias using the learned representation. \citet{ChengG022a} used VAE-based techniques for mediation analysis and proposed CMAVAE that can learn the representation of latent confounders to estimate direct and indirect effects under the sequential ignorability assumption~\cite{imai2010general}. 

However, CMAVAE assumes that there is only one type of latent confounders, which is not applicable to the different cases in practice. Our work improves CMAVAE. We extend the piecemeal deconfounding assumption to disentangled representation learning. The goal is to achieve more accurate estimation results and enhance the generalisation ability of the method.

\section{Conclusion}
\label{sec:Conclusion}
This work investigates an important causal mediation analysis problem, i.e., how to achieve accurate estimation of NDE, NIE and TE. We use a relaxed assumption, piecemeal deconfounding rather than sequential ignorability for causal mediation analysis. We extend the piecemeal deconfounding assumption to representation learning in a general case with three types of latent confounders. Then, we introduce DMAVAE, a VAE-based method that can learn disentangled representations from proxy attributes for estimating NDE, NIE and TE. With extensive experiments, we show that the proposed method significantly outperforms other mediation analysis methods on synthetic datasets. We demonstrate that DMAVAE has a strong generalisation ability and can be applied to different cases. Furthermore, we verify the ability of DMAVAE on a real-world application. For future work, we will apply this model to scenarios of interest in mediation analysis to help understand mechanisms or guide policy formulation.

\section{Acknowledgments}
This work has been partially supported by Australian Research Council (DP200101210), Natural Sciences and Engineering Research Council of Canada and University Presidents Scholarship (UPS) of University of South Australia.

\bibliography{DMAVAE}

\end{document}